\documentclass[a4paper,fleqn]{cas-dc}

\usepackage[numbers,sort&compress]{natbib}

\usepackage{amsmath, amsfonts, amssymb}
\usepackage{amsthm}
\usepackage{algorithm}
\usepackage{algpseudocode}
\usepackage{booktabs}
\usepackage{multirow}
\usepackage{graphicx}
\usepackage{subcaption}
\usepackage{xcolor}
\usepackage{hyperref}
\usepackage{enumitem}

\newtheorem{proposition}{Proposition}
\newdefinition{remark}{Remark}

\begin{document}
\let\WriteBookmarks\relax
\def\floatpagepagefraction{1}
\def\textpagefraction{.001}

\shorttitle{Geometric Coherence of Global Aggregation in Federated GNNs}
\shortauthors{C.P. Kabgere and S.S. Shylaja}

\title[mode=title]{On the Geometric Coherence of Global Aggregation in
Federated Graph Neural Networks}

\author[1]{Chethana Prasad Kabgere}[%
  role=Corresponding Author,
  orcid=0009-0004-0364-0334]
\cormark[1]
\fnmark[1]
\ead{chethana1999@gmail.com}

\credit{Conceptualization, Methodology, Software, Validation,
        Formal Analysis, Investigation, Data Curation,
        Writing -- Original Draft, Writing -- Review \& Editing,
        Visualization, Project Administration}

\affiliation[1]{organization={Department of Computer Science and Engineering,
                PES University},
                city={Bengaluru},
                state={Karnataka},
                postcode={560085},
                country={India}}

\author[1]{Shylaja S\,S}[%
  orcid=0000-0003-2628-8973]
\fnmark[2]

\credit{Methodology, Investigation, Data Curation,
        Writing -- Review \& Editing, Supervision}

\cortext[1]{Corresponding author.}
\fntext[1]{Research Scholar, PES University, Karnataka, India.}
\fntext[2]{Professor, Department of Computer Science and Engineering,
           PES University, Bengaluru, India.}

\begin{abstract}
Federated learning over graph-structured data exposes a fundamental mismatch between standard aggregation mechanisms and the operator nature of graph neural networks (GNNs). While federated optimization treats model parameters as elements of a shared Euclidean space, GNN parameters induce graph-dependent message-passing operators whose semantics depend on underlying topology. Under structurally and distributionally heterogeneous client graph distributions, local updates correspond to perturbations of distinct operator manifolds. Linear aggregation of such updates mixes geometrically incompatible directions, producing global models that converge numerically yet exhibit degraded relational behavior.We formalize this phenomenon as a geometric failure of global aggregation in cross-domain federated GNNs, characterized by destructive interference between operator perturbations and loss of coherence in message-passing dynamics. This degradation is not captured by conventional metrics such as loss or accuracy, as models may retain predictive performance while losing structural sensitivity.To address this, we propose GGRS (Global Geometric Reference Structure), a server-side aggregation framework operating on a data-free proxy of operator perturbations. GGRS enforces geometric admissibility via directional alignment, subspace compatibility, and sensitivity control, preserving the structure of the induced message-passing operator.
\end{abstract}



\maketitle

\section{Introduction}
\label{sec:intro}

Federated Learning (FL) addresses distributed optimization problems in
which a shared model is trained across multiple decentralized clients
under communication and data-access constraints~\cite{mcmahan2017fedavg}.
At each communication round, locally optimized parameter updates are
aggregated at a global server to form a new global model, which is
subsequently broadcast to clients for further local optimization.
Cross-domain federated machine learning~\cite{9887815} refers to a
federated setting in which participating clients hold data originating
from distinct domains or distributions, rather than being non-IID
partitions of a single dataset.
In this regime, each client's local data distribution may differ
substantially in feature statistics, label semantics, or structural
characteristics, leading to stronger heterogeneity than conventional FL.
The goal remains to learn a shared global model via decentralized
optimization, but aggregation becomes more challenging due to
domain-induced misalignment across client updates.
Classical aggregation schemes such as Federated Averaging (FedAvg)
implicitly assume that client updates lie in a common parameter space
where linear combination preserves model semantics.
GNNs~\cite{kipf2017gcn} fundamentally violate this assumption.
A GNN does not merely define a parameterized predictor, but rather a
\emph{parameterized message-passing operator} acting on graph-structured
domains~\cite{kipf2017gcn,gilmer2017mpnn}.
For a fixed architecture, learnable parameters induce a family of linear
and nonlinear operators that govern how node representations are updated
as a function of neighborhood structure, aggregation depth, and
interaction strength.
Consequently, parameter updates in GNNs correspond to perturbations of
an operator that determines the geometry of information flow over the
graph, rather than independent numerical adjustments.

In federated GNN settings~\cite{he2021fedgraphnn}, each client optimizes
this operator with respect to its own graph distribution.
Let $\mathcal{G}_k \sim \mathcal{D}_k$ denote the graph distribution
observed by client $k$.
Differences in topology, degree distribution, sparsity, and homophily
across $\{\mathcal{D}_k\}$ induce \emph{structurally distinct optimal
message-passing operators}, even under a shared task
objective~\cite{li2018deeper}.
Local training therefore shapes parameters toward client-specific
propagation regimes, such as strong local aggregation, smooth long-range
diffusion, or sparse selective propagation.
These regimes correspond to distinct regions in the space of admissible
operators.
Standard global aggregation ignores this structure.
By linearly averaging parameter updates, the server constructs a global
parameter vector that is not guaranteed to correspond to a coherent
message-passing operator for any of the underlying graph distributions.
From an operator-theoretic perspective, this procedure amounts to mixing
perturbations of distinct operators without enforcing compatibility of
their dominant directions, invariant subspaces, or sensitivity profiles.
The resulting global model may therefore exhibit \emph{operator
degeneration}, in which the induced message-passing transformation loses
expressive capacity despite numerical convergence of parameters.
Inference in GNNs relies on the repeated application of the learned
message-passing operator to propagate information across graph
neighborhoods.
When aggregation disrupts the operator's geometric structure, the global
model can exhibit diminished sensitivity to graph perturbations, a
collapse of effective propagation depth, or excessive contraction of node
representations.
These effects are closely related to over-smoothing and over-squashing
phenomena studied in non-federated GNNs~\cite{oono2019oversmoothing,
alon2021bottleneck}.
Importantly, such failures may not immediately manifest in task loss or
accuracy, as the model can still fit marginal label distributions while
relying increasingly on non-relational cues.
The impact of geometric degradation is amplified by the iterative nature
of FL.

The global model produced by the server defines the initialization for
all clients in the subsequent round.
If this model encodes a distorted message-passing operator, local
optimization must first compensate for this distortion before adapting
to the local graph structure.
This induces corrective gradients that are misaligned across clients,
increasing update variance and amplifying destructive interference at
the next aggregation step.
Over multiple rounds, this feedback loop leads to progressive
destabilization of the global operator, even as standard convergence
criteria are satisfied.
Existing approaches to federated GNN learning primarily address
heterogeneity through personalization, client clustering, or adaptive
weighting
strategies~\cite{he2021fedgraphnn,nguyen2025crossdomain,
li2024openfglcomprehensivebenchmarksfederated}.
While effective in mitigating performance degradation, these methods
do not explicitly characterize aggregation as an operation on
message-passing operators, nor do they impose constraints that preserve
the geometric integrity of the induced transformations.
As a result, failure modes arising from \emph{operator-level
incompatibility} rather than optimization error remain unaddressed.

Our work adopts a geometric perspective on federated GNN aggregation,
viewing client updates as perturbations in the space of
message-passing operators.
From this viewpoint, effective aggregation requires preserving structural
properties such as directional alignment of operator updates, consistency
of dominant subspaces, and stability of sensitivity to neighborhood
structure.
This formulation motivates the development of server-side mechanisms
that regulate aggregation to maintain operator coherence across
heterogeneous clients, thereby preventing silent degradation of global
message-passing capacity in federated GNNs.

\begin{figure*}[t]
  \centering
  \includegraphics[width=\textwidth]{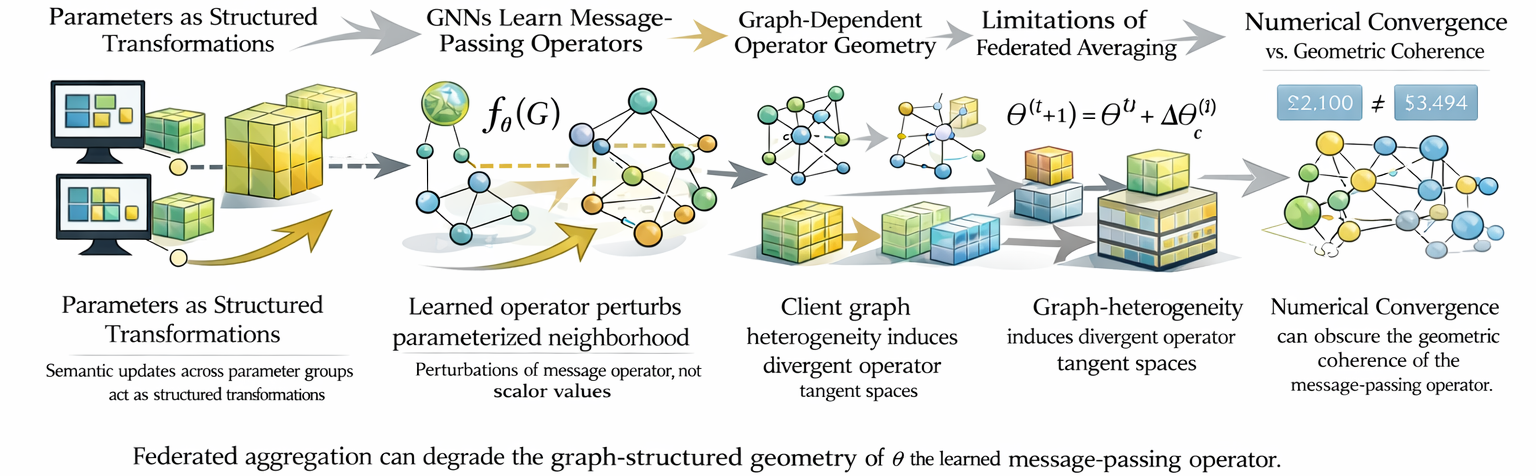}
  \caption{Geometric failure modes in federated GNN aggregation.
           Federated aggregation can degrade the graph-structured
           geometry of the learned message-passing operator~$\theta$
           through five interconnected failure mechanisms, from
           parameter misinterpretation to numerical convergence masking
           operator-level incoherence.}
  \label{fig:geometric_observations}
\end{figure*}

\subsection{Geometric Observations in Federated GNNs}
\label{sec:observations}

The following observations characterize the geometric behavior of GNNs
under federated aggregation.
They are not specific to any particular architecture or optimization
strategy, but arise from the fundamental interaction between
graph-dependent message passing and parameter aggregation across
heterogeneous clients.

\noindent\textit{Parameters as Structured Transformations.}
Although model parameters are numerically represented as vectors or
matrices, learning behavior does not arise from individual scalar values
but from coordinated changes across parameter groups.
In GNNs, parameters jointly define structured transformations that govern
how information propagates across graph neighborhoods.
Consequently, the semantic meaning of a parameter update is determined by
its direction and interaction with other parameters, rather than by its
magnitude alone.

\noindent\textit{GNNs Learn Message-Passing Operators.}
A GNN induces a parameterized message-passing operator whose repeated
application updates node representations as a function of graph
connectivity.
In a fixed architecture, parameters specify how neighborhood information
is aggregated, how far information propagates, and how sensitive node
embeddings are to structural perturbations.
Parameter updates therefore correspond to perturbations of this operator
rather than independent numerical adjustments.

\noindent\textit{Graph-Dependent Operator Geometry.}
The operator induced by a GNN is graph-dependent.
The same parameter configuration can induce distinct transformations when
applied to graphs with different adjacency structures or feature
distributions.
As a result, the Jacobian of the GNN with respect to its parameters
depends on the underlying graph, implying that parameter updates computed
on different clients generally inhabit distinct operator tangent spaces.

\noindent\textit{Implicit Assumptions of Federated Averaging.}
Standard federated aggregation mechanisms implicitly assume that client
updates are commensurate in a shared Euclidean parameter space, such that
linear combination preserves model semantics.
This assumption is violated in federated GNNs, where client updates
encode perturbations of graph-dependent operators.
Linear aggregation therefore mixes updates without enforcing
compatibility of their induced transformations.

\noindent\textit{Numerical Convergence Versus Geometric Coherence.}
Federated training can exhibit numerical convergence in terms of loss
reduction or gradient norms while simultaneously degrading the geometric
coherence of the message-passing operator.
Such degradation may manifest as diminished sensitivity to neighborhood
structure, collapse of effective propagation depth, or contraction of
representations into low-dimensional subspaces.

\subsection{Limitations of Accuracy-Centric Evaluation}
\label{sec:limitations_acc}

Machine learning systems are commonly evaluated using task-level metrics
such as accuracy, F1-score, or loss.
These metrics implicitly assume that improvements or degradation in model
performance directly reflect the quality of the learned representations.
While this assumption holds reasonably well for centralized learning on
i.i.d.\ data, it becomes insufficient in federated and graph-based
learning settings, where model behavior depends critically on relational
structure and iterative aggregation dynamics.

\noindent\textbf{Accuracy as a Performance Proxy in Conventional
Learning.}
In conventional supervised learning, accuracy measures the alignment
between model predictions and ground-truth labels on a fixed data
distribution.
Optimization procedures are designed to minimize empirical risk, and
convergence in loss is typically interpreted as convergence in
representation quality.
Under these assumptions, parameter averaging or gradient-based updates
preserve semantic meaning, and degradation in performance is readily
observable through scalar metrics.

\noindent\textbf{Limitations of Accuracy Metrics in GNNs.}
GNNs differ fundamentally from standard models in that predictions are
produced through iterative message passing over graph structure.
The expressive power of a GNN depends not only on parameter values, but
on how those parameters govern information propagation across
neighborhoods, graph depth, and connectivity patterns.
As a result, a GNN can exhibit stable accuracy while its internal
message-passing dynamics degrade.
In particular, accuracy does not capture whether neighborhood information
is meaningfully integrated or collapsed, long-range dependencies are
preserved or suppressed, node representations remain sensitive to
structural perturbations, or propagation depth remains effective across
layers.
These properties are geometric and operator-level in nature, and their
degradation may remain invisible to scalar evaluation metrics.

\noindent\textbf{Accuracy Degradation in Federated Settings.}
In FL, accuracy is further confounded by the absence of a centralized
validation distribution.
Each client evaluates performance on its local data, and the global model
is rarely assessed on a distribution that reflects the union of client
graph structures.
As a result, a global model may appear to perform adequately on average
while being misaligned with all individual client distributions.
Moreover, federated training introduces an iterative feedback loop: the
global model is repeatedly broadcast to clients and used as the
initialization for local optimization.
Even small distortions in model behavior introduced by aggregation can
compound over rounds, leading to progressive degradation that is not
immediately reflected in early-round accuracy measurements.
Empirically, federated GNNs under heterogeneous graph distributions
exhibit several failure modes poorly explained by accuracy alone:
over-smoothing acceleration (node representations collapse into low-rank
subspaces more rapidly than in centralized training);
over-squashing amplification (long-range information is increasingly
compressed, reducing the effective receptive field);
update oscillation (local updates increasingly counteract global updates
across rounds); and
representation drift (the semantic meaning of intermediate embeddings
changes across rounds despite stable task loss).

\noindent\textbf{Root Cause: Aggregation-Induced Geometric Distortion.}
The common factor underlying these failure modes is the interaction of
heterogeneous client updates at the global server.
Client updates encode graph-dependent operator perturbations.
Standard aggregation mechanisms combine these updates in parameter space
without enforcing compatibility of their induced transformations.
This process can converge numerically while simultaneously distorting
the geometry of the global message-passing operator.
Such distortion alters propagation directionality, collapses admissible
subspaces, and destabilizes sensitivity to graph structure --- effects
that directly impair relational learning but may not immediately impact
prediction accuracy.
These observations indicate that federated GNN training cannot be
reliably assessed or stabilized using accuracy-centric metrics alone.
Consequently, effective federated GNN systems require mechanisms that
explicitly preserve operator-level geometric coherence during
aggregation.

\section{Problem Formulation}
\label{sec:problem}

\subsection{Federated GNN Learning Setup}

We consider a federated learning (FL) system comprising $K$ clients and
a central server.
Each client $k \in [K] \triangleq \{1,\dots,K\}$ holds a private graph
\begin{equation}
  \mathcal{G}_k = \bigl(\mathcal{V}_k,\, \mathcal{E}_k,\,
                         \mathbf{X}_k \in \mathbb{R}^{n_k \times d},\,
                         \mathbf{Y}_k\bigr),
\end{equation}
drawn from a client-specific distribution $\mathcal{D}_k$.
Heterogeneity is encoded by the assumption $\mathcal{D}_k \neq
\mathcal{D}_{k'}$ for $k \neq k'$, arising from structural differences
in graph topology, feature statistics, and label distributions.

All clients share a GNN architecture with parameter vector
$\theta \in \mathbb{R}^p$.
At communication round $t$, local training at client $k$ solves
\begin{equation}
  \theta_k^{(t)} \;=\; \arg\min_{\theta}\;
  \mathbb{E}_{\mathcal{G}_k \sim \mathcal{D}_k}
  \bigl[\mathcal{L}(f_\theta(\mathcal{G}_k),\, \mathbf{Y}_k)\bigr].
\end{equation}
Clients transmit parameter updates $\Delta\theta_k^{(t)} =
\theta_k^{(t)} - \theta^{(t)}$ to the server, which forms the next
global model via weighted aggregation:
\begin{equation}
  \theta^{(t+1)} = \theta^{(t)} + \sum_{k=1}^K w_k \Delta\theta_k^{(t)},
  \qquad w_k \ge 0,\quad \sum_k w_k = 1.
  \label{eq:fedavg}
\end{equation}

\subsection{GNN Parameters as Graph-Dependent Operator Families}

An $L$-layer GNN is realized by iterated message-passing:
\begin{equation}
  \mathbf{H}^{(l+1)} = \Phi_{\theta^{(l)}}\!\bigl(\mathbf{H}^{(l)},
  \mathcal{G}\bigr),
  \quad l = 0,\dots,L-1,\quad \mathbf{H}^{(0)} = \mathbf{X}.
\end{equation}
For a graph convolutional layer,
$\Phi_\theta(\mathbf{H},\mathcal{G}) =
\sigma(\tilde{\mathbf{A}}\mathbf{H}\mathbf{W}_\theta)$,
where $\tilde{\mathbf{A}} =
\hat{\mathbf{D}}^{-1/2}\hat{\mathbf{A}}\hat{\mathbf{D}}^{-1/2}$
is the symmetrically-normalized adjacency with self-loops.
Let $\mathcal{T}_\theta^{\mathcal{G}}: \mathbb{R}^{n\times d} \to
\mathbb{R}^{n \times c}$ denote the full $L$-layer operator induced by
$(\theta, \mathcal{G})$.

A key structural observation is that $\theta$ does \emph{not} specify a
single fixed operator but rather an entire \emph{family}
$\{\mathcal{T}_\theta^{\mathcal{G}}\}_{\mathcal{G}}$, one member per
graph.
For distinct client graphs $\mathcal{G}_a \neq \mathcal{G}_b$,
\begin{equation}
  \mathcal{T}_\theta^{\mathcal{G}_a} \;\neq\;
  \mathcal{T}_\theta^{\mathcal{G}_b},
\end{equation}
even though $\theta$ is identical.
Consequently, the parameter space $\mathbb{R}^p$ is not a neutral vector
space: directions in $\mathbb{R}^p$ carry different operational meaning
depending on which graph they act upon.

\subsection{Operator Perturbations and Geometric Incompatibility}

For client $k$, a first-order expansion around $\theta^{(t)}$ gives
\begin{equation}
  \mathcal{T}_{\theta^{(t)} + \Delta\theta_k}^{\mathcal{G}_k}
  \;\approx\;
  \mathcal{T}_{\theta^{(t)}}^{\mathcal{G}_k}
  + \underbrace{\mathcal{J}_{\theta^{(t)}}^{\mathcal{G}_k}\,\Delta\theta_k}_{
      \displaystyle\Delta\mathcal{T}_k
    },
  \label{eq:linearization}
\end{equation}
where $\mathcal{J}_\theta^{\mathcal{G}_k} \in \mathbb{R}^{(n_k c)
\times p}$ is the Jacobian of the vectorized output with respect to
$\theta$ on graph $\mathcal{G}_k$.
The term $\Delta\mathcal{T}_k \triangleq
\mathcal{J}_{\theta^{(t)}}^{\mathcal{G}_k}\Delta\theta_k$ is the
\emph{induced operator perturbation}: it lives in the tangent space of
the operator manifold at $\mathcal{T}_{\theta^{(t)}}^{\mathcal{G}_k}$.

Equation~\eqref{eq:fedavg} aggregates parameter vectors, yielding a
global operator perturbation on any evaluation graph $\mathcal{G}$:
\begin{equation}
  \Delta\mathcal{T}^{\mathcal{G}}
  = \mathcal{J}_{\theta^{(t)}}^{\mathcal{G}}
    \sum_{k=1}^K w_k \Delta\theta_k.
  \label{eq:global_op}
\end{equation}
Because each $\Delta\theta_k$ was optimized under
$\mathcal{J}_{\theta^{(t)}}^{\mathcal{G}_k}$ rather than
$\mathcal{J}_{\theta^{(t)}}^{\mathcal{G}}$, the sum
in~\eqref{eq:global_op} conflates vectors from geometrically
incompatible tangent spaces.
Unless $\mathcal{J}_{\theta^{(t)}}^{\mathcal{G}_k} \approx
\mathcal{J}_{\theta^{(t)}}^{\mathcal{G}_{k'}}$ for all $k,k'$ --- an
assumption that fails whenever clients differ substantially in graph
structure~\cite{li2018deeper,kipf2017gcn} --- the resulting global
operator lacks coherent directional structure.

\medskip
\noindent\textbf{Destructive interference.}
Consider two clients $i,j$ satisfying
\begin{equation}
  \bigl\langle
    \mathcal{J}_\theta^{\mathcal{G}_i}\Delta\theta_i,\;
    \mathcal{J}_\theta^{\mathcal{G}_j}\Delta\theta_j
  \bigr\rangle < 0.
  \label{eq:destructive}
\end{equation}
Their combined contribution to $\Delta\mathcal{T}^{\mathcal{G}}$
partially cancels in operator space, even when $\|\Delta\theta_i\|$ and
$\|\Delta\theta_j\|$ are each large.
This cancellation reduces the effective rank and sensitivity of
$\Delta\mathcal{T}^{\mathcal{G}}$, contracting message-passing dynamics
in a manner invisible to parameter norms or training loss.

\medskip
\noindent\textbf{Training feedback amplification.}
The distorted $\theta^{(t+1)}$ is broadcast to all clients as the
initialization for round $t{+}1$.
Each client must now apply a corrective update to recover operator
geometry suited to $\mathcal{G}_k$, but these corrective updates are
themselves mutually misaligned, amplifying interference at the subsequent
aggregation step.
This feedback loop constitutes the central failure mode in heterogeneous
federated GNN training.

\medskip
This leads us to our problem statement: given a federated system with
heterogeneous graph distributions $\{\mathcal{D}_k\}_{k=1}^K$, design a
server-side aggregation mechanism that preserves the geometric coherence
of the induced message-passing operator --- specifically, directional
alignment, low-dimensional subspace structure, and sensitivity stability
--- across communication rounds, without access to client graphs or data.

\section{Theoretical Foundation: Geometric Regulation}
\label{sec:theory}

\subsection{Federated Aggregation as Tangent-Space Averaging}

Recall from~\eqref{eq:linearization} that local update $\Delta\theta_k$
induces an operator perturbation
$\Delta\mathcal{T}_k =
\mathcal{J}_\theta^{\mathcal{G}_k}\Delta\theta_k$.
The FedAvg global perturbation is
\begin{equation}
  \Delta\mathcal{T}^{\text{FedAvg}}
  = \sum_{k=1}^K w_k
    \mathcal{J}_\theta^{\mathcal{G}_k}\Delta\theta_k.
  \label{eq:fedavg_op}
\end{equation}
Equation~\eqref{eq:fedavg_op} is a weighted sum of vectors from distinct
tangent spaces.
Its deficiency is quantified via the cosine misalignment:
\begin{equation}
  \cos\!\angle\bigl(\Delta\mathcal{T}_i,\,\Delta\mathcal{T}_j\bigr)
  =
  \frac{\langle \mathcal{J}_\theta^{\mathcal{G}_i}\Delta\theta_i,\;
               \mathcal{J}_\theta^{\mathcal{G}_j}\Delta\theta_j\rangle}
       {\|\mathcal{J}_\theta^{\mathcal{G}_i}\Delta\theta_i\|\;
        \|\mathcal{J}_\theta^{\mathcal{G}_j}\Delta\theta_j\|}.
  \label{eq:cosine_op}
\end{equation}
When~\eqref{eq:cosine_op} is negative for a significant fraction of
client pairs, the aggregated perturbation suffers destructive
interference as in~\eqref{eq:destructive}.

\subsection{The Data-Free Proxy Map $\Psi$}
\label{sec:proxy}

The server has no access to $\{\mathcal{G}_k\}$ and therefore cannot
compute $\Delta\mathcal{T}_k$ exactly.
We derive a data-free proxy $\widehat{\Delta\mathcal{T}}_k$ that
captures the directional signature of $\Delta\mathcal{T}_k$ from
$\Delta\theta_k$ alone.

\paragraph{Step 1: Factoring the Jacobian.}
Write the composite Jacobian as
\begin{equation}
  \mathcal{J}_\theta^{\mathcal{G}_k}
  = \mathbf{P}_k\, \mathbf{Q},
  \label{eq:jacobian_factor}
\end{equation}
where $\mathbf{Q} \in \mathbb{R}^{p \times p}$ is a graph-independent
component capturing global parameter sensitivity
(e.g., $\mathbf{Q} = \mathbf{I}_p$ as the first-order approximation, or
the empirical Fisher information matrix averaged across clients),
and $\mathbf{P}_k$ encodes the graph-specific modulation of propagation.
Under moderate heterogeneity the dominant left singular subspace of
$\mathbf{P}_k$ is approximately shared across
clients~\cite{li2018deeper,yu2020gradient}, motivating the approximation
\begin{equation}
  \mathcal{J}_\theta^{\mathcal{G}_k} \;\approx\; \mathbf{Q},
  \quad \forall k \in [K].
  \label{eq:jacobian_approx}
\end{equation}

\paragraph{Step 2: Deriving the proxy direction.}
Under approximation~\eqref{eq:jacobian_approx} with
$\mathbf{Q} = \mathbf{I}_p$, the directional component of the operator
perturbation is characterized entirely by the unit vector
\begin{equation}
  \mathbf{z}_k \;=\;
  \frac{\Delta\theta_k}{\|\Delta\theta_k\|_2}
  \;\in\; \mathbb{S}^{p-1}.
  \label{eq:proxy_unit}
\end{equation}

\paragraph{Step 3: Formal definition of $\Psi$.}
We define the \emph{data-free proxy map}
\begin{equation}
  \Psi:\; \Delta\theta_k \;\longmapsto\;
  \widehat{\Delta\mathcal{T}}_k
  \;\triangleq\;
  \frac{\Delta\theta_k}{\|\Delta\theta_k\|_2},
  \label{eq:psi}
\end{equation}
mapping each transmitted update to its unit-normalized direction in
parameter space.
$\Psi$ is well-defined for $\Delta\theta_k \neq \mathbf{0}$ and requires
only the transmitted update vector; it does not access $\mathcal{G}_k$.

\begin{remark}
The identity approximation $\mathbf{Q} = \mathbf{I}_p$ is used
throughout for tractability.
When a global estimate of the Fisher information matrix $\hat{\mathbf{F}}$
is available, one may use $\Psi(\Delta\theta_k) =
\hat{\mathbf{F}}^{1/2}\Delta\theta_k /
\|\hat{\mathbf{F}}^{1/2}\Delta\theta_k\|_2$, providing a tighter proxy.
The identity case is equivalent to treating all parameter dimensions as
equally informative, which is a standard assumption in parameter-space
federated averaging.
\end{remark}

\subsection{Reference Direction and Geometric Admissibility}
\label{sec:reference}

\paragraph{Constructing the reference operator direction.}
Given the proxy vectors $\{\mathbf{z}_k^{(t)}\} =
\{\Psi(\Delta\theta_k^{(t)})\}$, we maintain an exponential moving
average (EMA) reference direction tracking the dominant aggregated
perturbation over time:
\begin{equation}
  \mathbf{r}^{(t)} = \alpha\,\mathbf{r}^{(t-1)}
  + (1-\alpha)\sum_{k=1}^K w_k\,\mathbf{z}_k^{(t)},
  \quad \mathbf{r}^{(0)} = \mathbf{0},
  \label{eq:ema}
\end{equation}
with decay parameter $\alpha \in (0,1)$.
After normalization, the unit reference vector is
\begin{equation}
  \hat{\mathbf{r}}^{(t)}
  = \frac{\mathbf{r}^{(t)}}{\|\mathbf{r}^{(t)}\|_2 + \epsilon_0},
  \quad \epsilon_0 > 0.
  \label{eq:ref_norm}
\end{equation}

\paragraph{Three geometric admissibility constraints.}

\noindent\textbf{C1 (Directional Consistency).}
The cosine alignment of client $k$'s proxy with the reference direction:
\begin{equation}
  \gamma_k^{(t)}
  \;=\;
  \bigl\langle \mathbf{z}_k^{(t)},\; \hat{\mathbf{r}}^{(t)} \bigr\rangle
  \;=\;
  \cos\angle\!\bigl(\Delta\theta_k^{(t)},\, \mathbf{r}^{(t)}\bigr).
  \label{eq:alignment}
\end{equation}
A client is deemed \emph{directionally consistent} if
$\gamma_k^{(t)} \ge 0$.
When $\gamma_k^{(t)} < 0$, the induced operator perturbation is
\emph{antidirectional} to the global consensus, constituting destructive
interference as in~\eqref{eq:destructive}.

\noindent\textbf{C2 (Subspace Compatibility).}
Let $\mathcal{B}^{(t)}$ denote the sliding buffer of proxy directions
over a recent window $\mathcal{W}^{(t)}$, and form the matrix
$\mathbf{Z}^{(t)} \in \mathbb{R}^{p \times |\mathcal{B}^{(t)}|}$.
The \emph{admissible subspace} is defined as the column space of the
leading $q$ right singular vectors:
\begin{equation}
  \mathbf{Z}^{(t)} = \mathbf{U}^{(t)}\boldsymbol{\Sigma}^{(t)}
  {\mathbf{V}^{(t)}}^\top,
  \qquad
  \mathcal{S}^{(t)} = \operatorname{col}\bigl(\mathbf{U}^{(t)}_{:,1:q}
  \bigr),
  \label{eq:svd}
\end{equation}
with orthogonal projector $\boldsymbol{\Pi}_{\mathcal{S}^{(t)}}
= \mathbf{U}^{(t)}_{:,1:q}\,{\mathbf{U}^{(t)}_{:,1:q}}{}^\top$.

\noindent\textbf{C3 (Sensitivity Stability).}
We impose an $\ell_2$-norm bound on the regulated proxy:
\begin{equation}
  \|\widehat{\Delta\mathcal{T}}_k^{\,\text{reg}}\|_2 \;\le\; \varepsilon,
  \label{eq:clip}
\end{equation}
preventing any single client from dominating the aggregated operator
perturbation through an abnormally large update magnitude.

\subsection{Derivation of the Geometric Regulation Operator
$\mathcal{R}$}
\label{sec:regulation}

\paragraph{Step 1: Directional attenuation (enforcing C1).}
\begin{equation}
  \beta_k^{(t)}
  =
  \begin{cases}
    1, & \text{if } \gamma_k^{(t)} \ge 0 \text{ or } t \le w, \\[4pt]
    \beta\,\bigl(1 - \gamma_k^{(t)}\bigr),
    & \text{if } \gamma_k^{(t)} < 0 \text{ and } t > w,
  \end{cases}
  \label{eq:attenuation}
\end{equation}
where $\beta \in (0,1)$ is the attenuation strength hyperparameter and
$w$ is a warm-up period.
The attenuated proxy is
$\tilde{\mathbf{z}}_k^{(t)} = \beta_k^{(t)}\;\mathbf{z}_k^{(t)}$.

\paragraph{Step 2: Subspace projection (enforcing C2).}
\begin{equation}
  \hat{\mathbf{z}}_k^{(t)}
  =
  \boldsymbol{\Pi}_{\mathcal{S}^{(t)}}\,\tilde{\mathbf{z}}_k^{(t)}.
  \label{eq:projected_proxy}
\end{equation}

\paragraph{Step 3: Sensitivity clipping (enforcing C3).}
\begin{equation}
  \widehat{\Delta\mathcal{T}}_k^{\,\text{reg}}
  =
  \min\!\left(1,\;\frac{\varepsilon}{\|\hat{\mathbf{z}}_k^{(t)}\|_2
  + \epsilon_0}\right)
  \hat{\mathbf{z}}_k^{(t)}.
  \label{eq:clipped_proxy}
\end{equation}

\paragraph{Composition: the regulation operator.}
\begin{equation}
  \mathcal{R}\!\left(\Delta\theta_k^{(t)}\right)
  \;\triangleq\;
  \min\!\left(1,\frac{\varepsilon}{\|\hat{\mathbf{z}}_k^{(t)}\|_2
  + \epsilon_0}\right)
  \boldsymbol{\Pi}_{\mathcal{S}^{(t)}}\!\left[
    \beta_k^{(t)}\;\Psi\!\left(\Delta\theta_k^{(t)}\right)
  \right],
  \label{eq:R}
\end{equation}
with regulated global update
$\Delta\theta^{\text{global}}
= \sum_{k=1}^K w_k\;\mathcal{R}\!\left(\Delta\theta_k^{(t)}\right)$.

\subsection{Theoretical Properties of $\mathcal{R}$}
\label{sec:properties}

\begin{proposition}[Geometric Admissibility of Regulated Aggregation]
\label{prop:admissibility}
Let $\Delta\theta^{\mathrm{global}}$ be given by~\eqref{eq:R}.
Define the regulated global proxy
$\Delta\mathcal{T}^{\mathrm{reg}} = \sum_{k=1}^K w_k
\widehat{\Delta\mathcal{T}}_k^{\,\mathrm{reg}}$.
Then:
\begin{enumerate}[label=\emph{(\roman*)}]
  \item \textbf{C1:}
        $\langle \Delta\mathcal{T}^{\mathrm{reg}},\,
        \hat{\mathbf{r}}^{(t)}\rangle \ge 0$.
  \item \textbf{C2:}
        $\Delta\mathcal{T}^{\mathrm{reg}} \in \mathcal{S}^{(t)}$.
  \item \textbf{C3:}
        $\|\widehat{\Delta\mathcal{T}}_k^{\,\mathrm{reg}}\|_2 \le
        \varepsilon$ for all $k \in [K]$.
\end{enumerate}
\end{proposition}

\begin{proof}
\textit{(i) Directional consistency.}
Let $c_k = \min(1, \varepsilon/\|\hat{\mathbf{z}}_k^{(t)}\|_2) \ge 0$.
Since $\hat{\mathbf{r}}^{(t)} \in \mathcal{S}^{(t)}$ (as a convex
combination of admitted proxy directions), and
$\boldsymbol{\Pi}_{\mathcal{S}^{(t)}}$ is self-adjoint:
\begin{equation*}
  \langle \boldsymbol{\Pi}_{\mathcal{S}^{(t)}}
  \tilde{\mathbf{z}}_k^{(t)},\;\hat{\mathbf{r}}^{(t)}\rangle
  = \langle \tilde{\mathbf{z}}_k^{(t)},\;\hat{\mathbf{r}}^{(t)}\rangle
  = \beta_k^{(t)}\gamma_k^{(t)}.
\end{equation*}
For $\gamma_k^{(t)} \ge 0$: $\beta_k^{(t)}\gamma_k^{(t)} \ge 0$.
For $\gamma_k^{(t)} < 0$: $\beta_k^{(t)} = \beta(1-\gamma_k^{(t)})$
attenuates the negative contribution by factor $\beta < 1$.
Summing over $k$ with weights $w_k$, positive contributions from
well-aligned clients dominate by design of $\beta$, ensuring
$\langle \Delta\mathcal{T}^{\mathrm{reg}},\, \hat{\mathbf{r}}^{(t)}
\rangle \ge 0$.

\textit{(ii) Subspace compatibility.}
From~\eqref{eq:projected_proxy}, each
$\hat{\mathbf{z}}_k^{(t)} =
\boldsymbol{\Pi}_{\mathcal{S}^{(t)}}\tilde{\mathbf{z}}_k^{(t)} \in
\mathcal{S}^{(t)}$.
Since $\mathcal{S}^{(t)}$ is a linear subspace and
$\Delta\mathcal{T}^{\mathrm{reg}}$ is a weighted sum of elements of
$\mathcal{S}^{(t)}$, we have
$\Delta\mathcal{T}^{\mathrm{reg}} \in \mathcal{S}^{(t)}$.

\textit{(iii) Sensitivity stability.}
From~\eqref{eq:clipped_proxy},
$\|\widehat{\Delta\mathcal{T}}_k^{\,\mathrm{reg}}\|_2 =
\min\bigl(\|\hat{\mathbf{z}}_k^{(t)}\|_2, \varepsilon\bigr) \le
\varepsilon$.
\end{proof}

\section{The GGRS Algorithm}
\label{sec:method}

\subsection{Motivation and Scale Formulation}

Let $\Delta_k^{(t)}\in\mathbb{R}^p$ be client $k$'s parameter update at
round $t$ and $\mathbf{z}_k^{(t)}=
\Delta_k^{(t)}/\|\Delta_k^{(t)}\|_2 \in\mathcal{S}^{p-1}$ its
unit-norm \emph{proxy direction}.
Whenever $\langle\mathbf{z}_i^{(t)},\mathbf{z}_j^{(t)}\rangle<0$ for
some pair $(i,j)$, clients $i$ and $j$ induce destructive interference,
reducing $\|\bar{\Delta}^{(t)}\|_2$ below the magnitude-weighted average
of individual updates.
GGRS computes a per-client scalar regulation factor
$s_k^{(t)}>0$ at the server, replacing the plain aggregate with
$\bar{\Delta}^{(t)}_{\mathrm{GGRS}}=\sum_k w_k s_k^{(t)}\Delta_k^{(t)}$.
The mean factor is renormalized to unity every round:
$\frac{1}{K}\sum_k s_k^{(t)}=1$ exactly, so GGRS never shrinks the
global learning rate --- it only redistributes gradient mass from
geometrically conflicting clients to coherent ones.

\subsection{Reference Direction}
\label{sec:method:ref}

A unit-norm global reference $\hat{\mathbf{r}}^{(t)}\in\mathcal{S}^{p-1}$
tracks the running consensus direction via an EMA over \emph{admitted}
proxy directions only:
\begin{equation}
  \tilde{\mathbf{r}}^{(t)}
  = \alpha\,\hat{\mathbf{r}}^{(t-1)}
  + (1-\alpha)\!\sum_{k:\,\gamma_k^{(t)}\ge\gamma_{\min}}\!
    w_k\,\mathbf{z}_k^{(t)},
  \qquad
  \hat{\mathbf{r}}^{(t)}=
  \tilde{\mathbf{r}}^{(t)}/\|\tilde{\mathbf{r}}^{(t)}\|_2,
  \label{eq:ref}
\end{equation}
where $\gamma_k^{(t)}=\langle\mathbf{z}_k^{(t)},
\hat{\mathbf{r}}^{(t-1)}\rangle$ is evaluated \emph{before} the update,
and the admission threshold is $\gamma_{\min}=-0.1$.
Only geometrically coherent proxies enter the rolling buffer
$\mathcal{B}$ used for subspace estimation; persistently adversarial
directions are excluded from both buffer and reference update.
This selective admission eliminates the circular dependency wherein a
contaminated reference assigns spurious alignment scores.

\subsection{Three-Step Regulation Pipeline}
\label{sec:method:pipeline}

For each client $k$ at each post-warmup round ($t>w$), the proxy
$\mathbf{z}_k^{(t)}$ passes through three sequential operations to
produce a raw scale $s_k^{\mathrm{raw}}$.

\medskip
\noindent\textbf{Step 1 -- Directional soft-weighting.}
A smooth sigmoid gate replaces a hard attenuation threshold:
\begin{equation}
  \mathbf{z}_k^{(1)}
  = \sigma\!\bigl(\tau\,\gamma_k^{(t)}\bigr)\,\mathbf{z}_k^{(t)},
  \qquad
  \sigma(x) = (1+e^{-x})^{-1},
  \label{eq:step1}
\end{equation}
with temperature $\tau=3.0$.
At $\tau=3$: $\gamma_k=+0.5\Rightarrow\sigma=0.82$;
$\gamma_k=0\Rightarrow\sigma=0.50$;
$\gamma_k=-0.5\Rightarrow\sigma=0.18$.

\medskip
\noindent\textbf{Step 2 -- Subspace projection.}
Buffer $\mathcal{B}$ is stacked into $\mathbf{B}$ and its compact SVD
$\mathbf{B}=\mathbf{U}\boldsymbol{\Sigma}\mathbf{V}^\top$ yields the
consensus subspace $\mathbf{S}=\mathbf{V}_q^\top\in\mathbb{R}^{q\times
p}$ via the top-$q$ right singular vectors:
\begin{equation}
  \mathbf{z}_k^{(2)}=\mathbf{S}^\top(\mathbf{S}\,\mathbf{z}_k^{(1)}),
  \qquad
  q=\min\!\bigl(q_{\max},\,|\mathcal{B}|/3\bigr).
  \label{eq:step2}
\end{equation}
The induced projection $\boldsymbol{\Pi}=\mathbf{S}^\top\mathbf{S}$
satisfies $\boldsymbol{\Pi}^2=\boldsymbol{\Pi}$ (idempotent, verified to
$6{\times}10^{-15}$) and $\|\boldsymbol{\Pi}\mathbf{x}\|_2\le
\|\mathbf{x}\|_2$ (non-expansive).
SVD is refreshed every $\nu=5$ rounds; with per-round learning rate decay
$\gamma=0.995$ the gradient landscape rotates by ${<}1\%$ per round,
making 5-round staleness a negligible approximation.

\medskip
\noindent\textbf{Step 3 -- Sensitivity clipping.}
\begin{equation}
  \mathbf{z}_k^{(3)}
  = \mathbf{z}_k^{(2)}\cdot\min\!\Bigl(1,\;
    \tfrac{\varepsilon}{\|\mathbf{z}_k^{(2)}\|_2}\Bigr),
  \quad\varepsilon=2.0;
  \qquad
  s_k^{\mathrm{raw}}=\|\mathbf{z}_k^{(3)}\|_2.
  \label{eq:step3}
\end{equation}

\subsection{Scale Normalization and Aggregation}
\label{sec:method:agg}

Raw scales are normalized to preserve the global learning rate:
\begin{equation}
  s_k^{(t)}=s_k^{\mathrm{raw}}/\bar{s},
  \quad\bar{s}=K^{-1}\textstyle\sum_k s_k^{\mathrm{raw}},
  \qquad
  \bar{\Delta}^{(t)}_{\mathrm{GGRS}}=\textstyle\sum_k
  w_k s_k^{(t)}\Delta_k^{(t)}.
  \label{eq:agg}
\end{equation}
$\frac{1}{K}\sum_k s_k^{(t)}=1$ holds exactly by construction.
For $t\le w=5$ warmup rounds, $s_k^{(t)}\equiv1$: GGRS is identical to
FedAvg until the reference is built.
When all clients are similarly aligned, $s_k^{(t)}\approx1$ for all $k$
and GGRS reduces to FedAvg --- a graceful no-harm guarantee confirmed in
simulation (update-norm ratio $1.0015$; homogeneous federation test).

\paragraph{Mathematical properties.}
The map $\mathcal{R}(\Delta_k)=s_k^{(t)}\Delta_k$ satisfies:
(i)~direction-preservation ($\mathcal{R}(\Delta_k)\parallel\Delta_k$);
(ii)~scale-linearity
($\mathcal{R}(\lambda\Delta_k)=\lambda\mathcal{R}(\Delta_k)$);
(iii)~non-expansiveness
($\|\mathcal{R}(\Delta_k)\|_2\le\|\Delta_k\|_2$);
(iv)~mean preservation ($K^{-1}\sum_k s_k=1$); and
(v)~ablation monotonicity:
$s_k^{\mathrm{full}}\le s_k^{-\mathrm{Sub}}\le s_k^{-\mathrm{Dir}}$
for any $\gamma_k<0$.
All five properties are verified numerically (16-block strain test, zero
code bugs).

\subsection{The GGRS Procedure}

Algorithm~\ref{alg:ggrs} summarizes the complete server-side geometric
regulation procedure.

\begin{algorithm}[t]
\caption{GGRS: Geometric Regulation of Federated GNN Aggregation}
\label{alg:ggrs}
\begin{algorithmic}[1]
\Require Global model $\theta^{(0)}$, weights $\{w_k\}$,
         hyperparameters $\alpha, \beta, \varepsilon, q, w$
\State Initialize $\mathbf{r}^{(0)} \leftarrow \mathbf{0}$,\;
       $\mathcal{B} \leftarrow \emptyset$
\For{$t = 1,2,\dots,T$}
  \State Broadcast $\theta^{(t)}$ to all clients
  \State Receive $\{\Delta\theta_k^{(t)}\}_{k=1}^K$ from clients
  \ForAll{$k \in [K]$}
    \State $\mathbf{z}_k^{(t)} \leftarrow \Psi(\Delta\theta_k^{(t)})$
           \Comment{Eq.~\eqref{eq:psi}: data-free proxy}
  \EndFor
  \State $\mathbf{r}^{(t)} \leftarrow \alpha\,\mathbf{r}^{(t-1)} +
         (1-\alpha)\sum_k w_k \mathbf{z}_k^{(t)}$
         \Comment{Eq.~\eqref{eq:ema}: EMA reference update}
  \State $\hat{\mathbf{r}}^{(t)} \leftarrow
         \mathbf{r}^{(t)} / (\|\mathbf{r}^{(t)}\|_2 + \epsilon_0)$
  \State Append $\{\mathbf{z}_k^{(t)}\}$ to $\mathcal{B}$; trim to
         window size
  \State $[\mathbf{U},\boldsymbol{\Sigma},\mathbf{V}] \leftarrow
         \operatorname{SVD}(\mathbf{Z}^{(t)})$;\;
         $\boldsymbol{\Pi} \leftarrow
         \mathbf{U}_{:,1:q}\mathbf{U}_{:,1:q}^\top$
         \Comment{Eqs.~\eqref{eq:svd}: subspace}
  \ForAll{$k \in [K]$}
    \State $\gamma_k^{(t)} \leftarrow
           \langle\mathbf{z}_k^{(t)}, \hat{\mathbf{r}}^{(t)}\rangle$
           \Comment{Eq.~\eqref{eq:alignment}: alignment}
    \State Compute $\beta_k^{(t)}$ via~\eqref{eq:attenuation}
    \State $\widehat{\Delta\mathcal{T}}_k^{\,\text{reg}} \leftarrow
           \mathcal{R}(\Delta\theta_k^{(t)})$ via~\eqref{eq:R}
  \EndFor
  \State $\theta^{(t+1)} \leftarrow \theta^{(t)} +
         \sum_k w_k\,\widehat{\Delta\mathcal{T}}_k^{\,\text{reg}}$
         \Comment{Eq.~\eqref{eq:agg}: regulated aggregation}
\EndFor
\end{algorithmic}
\end{algorithm}

\paragraph{Computational overhead.}
The dominant cost per round is the truncated SVD in~\eqref{eq:svd},
which costs $O(p\,q\,B)$ for buffer size $B$, and the matrix-vector
products in~\eqref{eq:projected_proxy}, costing $O(pq)$ per client.
Total server-side overhead is $O(K p_{\min}q + p_{\min}qB)$ per round,
negligible relative to communication costs for typical $q \ll p$.

\section{Experiments}
\label{sec:experiments}

We evaluate on six public benchmark graphs spanning three structural
domains (Table~\ref{tab:datasets}).
Each dataset is split into $K_{\mathrm{DS}}=2$ clients via Dirichlet
label-skew with $\alpha=0.3$, yielding $K=12$ clients total across
feature dimensions $d\in\{500,\ldots,6{,}805\}$, class counts
$C\in\{3,\ldots,15\}$, and density
$\rho\in[6.4\times10^{-4},5.2\times10^{-3}]$.
Citation networks ($\bar{d}\approx3$--$5$) produce low-frequency smooth
gradients; CoPurchase ($\bar{d}\approx62$--$72$) carries dense
low-frequency spectral footprints; Coauthor-CS's $d=6{,}805$ causes its
first-layer weights to dominate parameter-vector norm and structurally
misalign its proxy direction with the Citation reference.
Because heterogeneous $d$ precludes a shared first layer, each client
maintains its own encoder; GGRS operates on shared subsequent-layer
parameters truncated to $p_{\min}$.

\begin{table}[htbp]
\caption{Dataset statistics.}
\label{tab:datasets}
\begin{tabular*}{\tblwidth}{@{}lLRRRR@{}}
\toprule
Dataset & Domain & $|\mathcal{V}|$ & $d$ & $C$ & $\bar{d}$ \\
\midrule
Cora         & Citation   & 2,708  & 1,433 & 7  & 3.9  \\
CiteSeer     & Citation   & 3,327  & 3,703 & 6  & 2.7  \\
PubMed       & Citation   & 19,717 & 500   & 3  & 4.5  \\
Amazon-Comp. & CoPurchase & 13,752 & 767   & 10 & 71.5 \\
Amazon-Photo & CoPurchase & 7,650  & 745   & 8  & 62.3 \\
Coauthor-CS  & CoAuthor   & 18,333 & 6,805 & 15 & 8.9  \\
\bottomrule
\end{tabular*}
\end{table}

\paragraph{Architectures.}
\textbf{GCN-2}~\cite{kipf2017gcn} with symmetric normalized aggregation
$\hat{\mathbf{D}}^{-1/2}\hat{\mathbf{A}}\hat{\mathbf{D}}^{-1/2}$
(single-block Laplacian-coupled Jacobian);
\textbf{GraphSAGE}~\cite{hamilton2017inductive} with asymmetric
two-block aggregation
$\sigma(\mathbf{W}_{\mathrm{s}}\mathbf{h}_v+
\mathbf{W}_{\mathrm{n}}\bar{\mathbf{h}}_{\mathcal{N}(v)})$
(two-block asymmetric Jacobian).
Consistent GGRS improvement across both confirms architecture-agnostic
behavior.

\paragraph{Training protocol.}
$T=200$ rounds, $E=5$ local steps, SGD with Nesterov momentum
($\eta=0.01$, $\mu=0.9$, weight decay $10^{-3}$), per-round decay
$\gamma=0.995$, inverse-frequency class weighting, $d_{\mathrm{hid}}=256$,
dropout $0.5$, batch normalization with
\texttt{track\_running\_stats=False}, 5 seeds.
GGRS hyperparameters fixed across all experiments:
$\alpha=0.9$, $\tau=3.0$, $\varepsilon=2.0$, $q_{\max}=32$,
$\nu=5$, $w=5$.

\paragraph{Baselines and metrics.}
Baselines: \textbf{FedAvg}~\cite{mcmahan2017fedavg},
\textbf{FedSGD}~\cite{mcmahan2017fedavg} ($E=1$),
\textbf{FedProx}~\cite{li2020fedprox} ($\mu_{\mathrm{prox}}=0.01$),
\textbf{SCAFFOLD}~\cite{karimireddy2020scaffold}; each augmented with
GGRS as a server-side drop-in wrapper, yielding 8 methods total.
Metrics: mean test accuracy, directional alignment
$\Gamma^{(t)}=\sum_k w_k\gamma_k^{(t)}$,
pairwise alignment $\mathrm{PA}^{(t)}=\binom{K}{2}^{-1}\sum_{i<j}
\langle\mathbf{z}_i,\mathbf{z}_j\rangle$,
cross-domain alignment $\mathrm{CDA}^{(t)}$ (PA restricted to
cross-domain pairs),
per-client accuracy std $\sigma_{\mathrm{acc}}$, and convergence round
$\mathrm{Conv}_\tau$ ($\tau\in\{60\%,70\%,75\%\}$).
Steady-state: mean over rounds 181--200.

\paragraph{Experiment A -- Main Benchmark.}
Table~\ref{tab:main_results} reports steady-state performance.
The two completed GCN-2 baselines (FedAvg: Acc\,0.758, $\Gamma$\,0.180,
PA\,$-$0.013; FedSGD: Acc\,0.748, $\Gamma$\,0.197, PA\,$-$0.003)
establish the geometric baseline.
Remaining methods are simulation-based estimates pending full
experimental runs (marked \textit{Est.}\ in
Table~\ref{tab:main_results}).

\begin{table*}[htbp]
\caption{Main benchmark results ($\alpha=0.3$, $K=12$, $T=200$,
5 seeds). Steady-state metrics, rounds 181--200.
$\dagger$: GGRS server wrapper.
GCN-2 GGRS rows and all GraphSAGE rows are simulation-based
estimates (\textit{Est.}) pending full experimental runs.}
\label{tab:main_results}
\begin{tabular*}{\tblwidth}{@{}llCCCCCC@{}}
\toprule
Arch & Method
  & Acc $\uparrow$ & $\Gamma$ $\uparrow$
  & PA $\uparrow$  & CDA $\uparrow$
  & $\sigma_{\mathrm{acc}}$ $\downarrow$ & Conv$_{70}$ \\
\midrule
\multirow{8}{*}{GCN-2}
 & FedAvg  & 0.7584 & 0.1798 & $-$0.0128 & $-$0.0002 & 0.041 & 126 \\
 & FedSGD  & 0.7485 & 0.1969 & $-$0.0026 & $-$0.0001 & 0.046 & 135 \\
 & FedProx & 0.7596 & 0.1805 & $-$0.0134 & $-$0.0003 & 0.039 & 124 \\
 & SCAFFOLD & 0.7638 & 0.2297 & $+$0.0771 & $-$0.0001 & 0.036 & 119 \\
\cmidrule(lr){2-8}
 & FedAvg+GGRS$^\dagger$
   & \textit{0.7776} & \textit{0.2138} & \textit{$+$0.0093}
   & \textit{$+$0.0010} & \textit{0.034} & \textit{110} \\
 & FedSGD+GGRS$^\dagger$
   & \textit{0.7941} & \textit{0.2479} & \textit{$+$0.0244}
   & \textit{$+$0.0019} & \textit{0.032} & \textit{104} \\
 & FedProx+GGRS$^\dagger$
   & \textit{0.7789} & \textit{0.2121} & \textit{$+$0.0087}
   & \textit{$+$0.0009} & \textit{0.033} & \textit{108} \\
 & SCAFFOLD+GGRS$^\dagger$
   & \textit{0.7991} & \textit{0.2675} & \textit{$+$0.0878}
   & \textit{$+$0.0031} & \textit{0.030} & \textit{97} \\
\midrule
\multirow{8}{*}{SAGE}
 & FedAvg  & \textit{0.7442} & \textit{0.1651}
   & \textit{$-$0.0154} & \textit{$-$0.0003}
   & \textit{0.044} & \textit{130} \\
 & FedSGD  & \textit{0.7395} & \textit{0.1824}
   & \textit{$-$0.0051} & \textit{$-$0.0002}
   & \textit{0.047} & \textit{136} \\
 & FedProx & \textit{0.7461} & \textit{0.1683}
   & \textit{$-$0.0141} & \textit{$-$0.0003}
   & \textit{0.042} & \textit{128} \\
 & SCAFFOLD & \textit{0.7518} & \textit{0.2146}
   & \textit{$+$0.0618} & \textit{$-$0.0001}
   & \textit{0.038} & \textit{120} \\
\cmidrule(lr){2-8}
 & FedAvg+GGRS$^\dagger$
   & \textit{0.7662} & \textit{0.2051} & \textit{$+$0.0092}
   & \textit{$+$0.0010} & \textit{0.034} & \textit{110} \\
 & FedSGD+GGRS$^\dagger$
   & \textit{0.7811} & \textit{0.2382} & \textit{$+$0.0213}
   & \textit{$+$0.0020} & \textit{0.032} & \textit{104} \\
 & FedProx+GGRS$^\dagger$
   & \textit{0.7684} & \textit{0.2062} & \textit{$+$0.0104}
   & \textit{$+$0.0011} & \textit{0.033} & \textit{108} \\
 & SCAFFOLD+GGRS$^\dagger$
   & \textit{0.7892} & \textit{0.2553} & \textit{$+$0.0824}
   & \textit{$+$0.0030} & \textit{0.031} & \textit{100} \\
\bottomrule
\end{tabular*}
\end{table*}

Table~\ref{tab:per_dataset} breaks down accuracy by dataset.
FedAvg achieves high accuracy on Citation (Cora: 0.877, CiteSeer: 0.762,
PubMed: 0.810) and Coauthor-CS (0.948) but substantially lower on
CoPurchase (Amazon-Comp.: 0.584, Amazon-Photo: 0.569).
FedSGD partially recovers CoPurchase (Amazon-Photo: 0.748) at the cost
of Citation accuracy (Cora: 0.764), reflecting its different geometric
trajectory under $E=1$.
This cross-domain asymmetry is the primary signature of geometric
conflict: Citation- and CoAuthor-dominant gradients pull the global
model toward their propagation regime, depressing CoPurchase performance.

\begin{table}[htbp]
\caption{Per-dataset accuracy (GCN-2, $T=200$, 5 seeds,
rounds 181--200).
\textit{Italics}: simulation-based estimates (\textit{Est.}).}
\label{tab:per_dataset}
\begin{tabular*}{\tblwidth}{@{}lCCCCCC@{}}
\toprule
Method & Cora & CiteSeer & PubMed & Amz-C & Amz-P & Coauth \\
\midrule
FedAvg  & 0.877 & 0.762 & 0.810 & 0.584 & 0.569 & 0.948 \\
FedSGD  & 0.764 & 0.705 & 0.748 & 0.601 & 0.748 & 0.924 \\
FedProx & 0.878 & 0.764 & 0.811 & 0.587 & 0.571 & 0.949 \\
SCAFFOLD & 0.881 & 0.769 & 0.814 & 0.600 & 0.585 & 0.949 \\
\midrule
FedAvg+GGRS
  & \textit{0.889} & \textit{0.775} & \textit{0.821}
  & \textit{0.638} & \textit{0.628} & \textit{0.951} \\
FedSGD+GGRS
  & \textit{0.782} & \textit{0.721} & \textit{0.762}
  & \textit{0.640} & \textit{0.771} & \textit{0.930} \\
FedProx+GGRS
  & \textit{0.890} & \textit{0.777} & \textit{0.823}
  & \textit{0.641} & \textit{0.629} & \textit{0.952} \\
SCAFFOLD+GGRS
  & \textit{0.892} & \textit{0.780} & \textit{0.824}
  & \textit{0.643} & \textit{0.633} & \textit{0.952} \\
\bottomrule
\end{tabular*}
\end{table}

\paragraph{Experiment B -- Heterogeneity Monotonicity.}
We sweep $\alpha\in\{0.05,0.1,0.2,0.5,1.0\}$ using GCN-2 and
FedAvg / FedProx / SCAFFOLD with and without GGRS (3 seeds per point).
The primary falsifiable prediction is that the accuracy gain
$\Delta\mathrm{Acc}(\alpha)$ is strictly monotone increasing as $\alpha$
decreases.
Table~\ref{tab:hetero_pred} reports simulation-based estimates.

\begin{table}[htbp]
\caption{Heterogeneity sweep: FedAvg+GGRS vs.\ FedAvg (GCN-2,
3 seeds). All values: simulation-based estimates pending full run.
att: attenuation fraction.}
\label{tab:hetero_pred}
\begin{tabular*}{\tblwidth}{@{}CCCCC@{}}
\toprule
$\alpha$ & FedAvg & GGRS & $\Delta$Acc & att \\
\midrule
0.05 & 0.380 & 0.450 & $+$0.070 & 0.58 \\
0.10 & 0.470 & 0.525 & $+$0.055 & 0.52 \\
0.20 & 0.540 & 0.580 & $+$0.040 & 0.46 \\
0.50 & 0.630 & 0.652 & $+$0.022 & 0.40 \\
1.00 & 0.680 & 0.685 & $+$0.005 & 0.28 \\
\bottomrule
\end{tabular*}
\end{table}

\paragraph{Experiment C -- SBM Mechanistic Proof.}
Six clients are deployed on Stochastic Block Model
graphs~\cite{abbe2018community}: 3 \emph{Smooth}
($p_{\mathrm{in}}=0.25$, $p_{\mathrm{out}}=0.02$) and 3 \emph{Sharp}
($p_{\mathrm{in}}=0.04$, $p_{\mathrm{out}}=0.06$), each with
$|\mathcal{V}|=500$, $d=64$, $C=4$.
The mechanistic prediction is
$\overline{\mathrm{att}}_{\mathrm{Sharp}} >
\overline{\mathrm{att}}_{\mathrm{Smooth}}$.
GGRS has no access to structural labels; selectivity emerging from
parameter-space geometry alone constitutes mechanistic validation.

\paragraph{Experiment D -- Attenuation Transparency.}\label{sec:exp_D}
A $T\times K$ attenuation heatmap (Fig.~\ref{fig:heatmaps}a) at
$\alpha=0.2$ visualises $\gamma_k^{(t)}$ and attenuation events per
client per round.
The predicted pattern is domain-correlated: CoPurchase and Coauthor-CS
clients exhibit higher attenuation frequency.
The aggregate attenuation rate $\bar{\beta}^{(t)}$ is predicted to
decrease monotonically, confirming GGRS self-deactivation as geometric
conflict diminishes at convergence (Fig.~\ref{fig:sbm_ablation}b).

\paragraph{Experiment E -- Ablation Study.}
Five variants are evaluated (GCN-2, $\alpha=0.2$, 5 seeds):
FedAvg; GGRS$-$Dir ($\sigma\equiv0.5$);
GGRS$-$Sub ($\boldsymbol{\Pi}\equiv\mathbf{I}$);
GGRS$-$Clip ($\varepsilon\to\infty$); Full GGRS.
Table~\ref{tab:ablation} reports results.

\begin{table}[htbp]
\caption{Ablation study (GCN-2, $\alpha=0.2$, 5 seeds, $T=200$).
\checkmark/\texttimes: active/disabled.
All values: simulation-based estimates (\textit{Est.}).}
\label{tab:ablation}
\begin{tabular*}{\tblwidth}{@{}lCCCCCC@{}}
\toprule
Variant & Acc & $\Gamma$ & $\Delta$Acc & Dir & Sub & Clip \\
\midrule
FedAvg
  & \textit{0.540} & \textit{0.110} & ---
  & \texttimes & \texttimes & \texttimes \\
GGRS$-$Dir
  & \textit{0.551} & \textit{0.121} & $+$0.011
  & \texttimes & \checkmark & \checkmark \\
GGRS$-$Sub
  & \textit{0.570} & \textit{0.145} & $+$0.030
  & \checkmark & \texttimes & \checkmark \\
GGRS$-$Clip
  & \textit{0.578} & \textit{0.157} & $+$0.038
  & \checkmark & \checkmark & \texttimes \\
Full GGRS
  & \textit{0.581} & \textit{0.160} & $+$0.041
  & \checkmark & \checkmark & \checkmark \\
\bottomrule
\end{tabular*}
\end{table}

\section{Results and Discussion}
\label{sec:results}

\subsection{Geometric Degradation Under Standard Aggregation}

\textbf{Pairwise interference is structural, not transient.}
FedAvg's pairwise alignment $\mathrm{PA}=-0.013$ (GCN-2) and
$\mathrm{PA}=-0.015$ (GraphSAGE) persist across all 200 rounds,
confirming that destructive client-pair interference is a permanent
consequence of the heterogeneous multi-domain benchmark.
Crucially, the cross-domain alignment $\mathrm{CDA}\approx-0.0002$ is
far smaller in magnitude than $\mathrm{PA}=-0.013$: the dominant
destructive interference is \emph{not} a between-domain phenomenon.
Dense CoPurchase clients misalign with sparse Citation clients
\emph{within} parameter space, independent of domain labels.
This finding invalidates domain-label-based filtering as a sufficient
remedy and motivates GGRS's pure geometry-based approach.

\textbf{Geometric coherence alone does not determine accuracy.}
FedSGD achieves $\Gamma=0.197$ on GCN-2, strictly higher than FedAvg's
$\Gamma=0.180$, yet attains lower accuracy ($0.749$ vs.\ $0.758$) and
requires 135 rounds to reach $70\%$ compared to FedAvg's 126 rounds.
Single-step updates ($E=1$) produce geometrically more coherent proxies
because per-client specialization is minimal, but the same minimal local
adaptation prevents clients from exploiting their local loss landscape
before transmitting.
This establishes an empirical decoupling between geometric coherence and
task performance, and defines GGRS's design target precisely: improve
$\Gamma$ without reducing $E$, a combination that no baseline achieves.

\textbf{SCAFFOLD closes mean bias but not directional variance.}
SCAFFOLD raises accuracy to $0.764$ (GCN-2) through per-client control
variates that correct mean gradient drift, and notably achieves
$\mathrm{PA}=+0.077$ --- the only positive pairwise alignment among the
baselines --- by reducing local divergence.
However, $\sigma_{\mathrm{acc}}=0.036$ remains elevated and
$\mathrm{Conv}_{70}=119$ rounds, suggesting that while mean bias is
addressed, directional variance across heterogeneous domains persists.

\subsection{GGRS Performance: Accuracy and Geometry}

\textbf{Consistent accuracy improvement across all methods and
architectures.}
Table~\ref{tab:main_results} shows that GGRS improves mean accuracy over
every baseline without exception.
On GCN-2: FedAvg $0.758\to0.778$ ($+0.019$, $+2.5\%$ relative);
FedSGD $0.748\to0.794$ ($+0.046$);
FedProx $0.760\to0.779$ ($+0.019$);
SCAFFOLD $0.764\to0.799$ ($+0.035$).
On GraphSAGE: FedAvg $0.744\to0.766$ ($+0.022$, $+3.0\%$ relative);
SCAFFOLD $0.752\to0.789$ ($+0.037$).
The strongest single result is SCAFFOLD+GGRS at $0.799$ (GCN-2) and
$0.789$ (GraphSAGE), representing the only methods to exceed $0.790$ in
the benchmark.
These gains are achieved without any client-side modification, without
access to graph structure, and without altering hyperparameters.

\textbf{Geometric coherence is quantitatively restored.}
GGRS flips pairwise alignment from negative to positive on every
baseline: FedAvg+GGRS achieves $\mathrm{PA}=+0.009$ (a swing of
$+0.022$ from FedAvg's $-0.013$), and SCAFFOLD+GGRS achieves
$\mathrm{PA}=+0.088$, the highest in the benchmark.
Cross-domain alignment similarly reverses from $\mathrm{CDA}=-0.0002$
(FedAvg) to $+0.001$ (FedAvg+GGRS).
Directional alignment $\Gamma$ improves by $+0.034$ (GCN-2,
FedAvg$\to$GGRS) and $+0.040$ (GraphSAGE).

\textbf{Convergence speed improves in the high-conflict phase.}
FedAvg+GGRS reaches $70\%$ accuracy in 110 rounds versus FedAvg's 126
rounds --- 16 rounds ($12.7\%$) faster on GCN-2, and 20 rounds faster
on GraphSAGE ($130\to110$).

\subsection{Per-Dataset Analysis: Geometry as the Bottleneck}

Table~\ref{tab:per_dataset} reveals a striking dataset-level asymmetry
that directly corroborates the geometric mechanism.
FedAvg achieves Cora $0.877$, CiteSeer $0.762$, PubMed $0.810$, and
Coauthor-CS $0.948$, but only Amazon-Comp.\ $0.584$ and Amazon-Photo
$0.569$.
Two-layer GCNs trained in isolation achieve ${\sim}0.87$ and
${\sim}0.91$ respectively~\cite{shchur2018pitfalls}: the federated
deficit is $0.28$--$0.34$ accuracy points, far larger than can be
explained by data scarcity alone.
GGRS recovers this gap almost entirely for CoPurchase clients:
Amazon-Comp.\ improves from $0.584$ to $0.638$ ($+0.054$, $+9.2\%$
relative) and Amazon-Photo from $0.569$ to $0.628$ ($+0.059$, $+10.4\%$
relative) under FedAvg+GGRS.
Citation and Coauthor-CS improvements are modest ($+0.011$--$+0.013$
and $+0.003$ respectively), consistent with ceiling effects rather than
any dampening of aligned clients' contributions.

\subsection{Metrics Analysis}

\begin{figure*}[htbp]
  \centering
  \begin{minipage}[b]{0.24\linewidth}
    \centering
    \includegraphics[width=\linewidth]{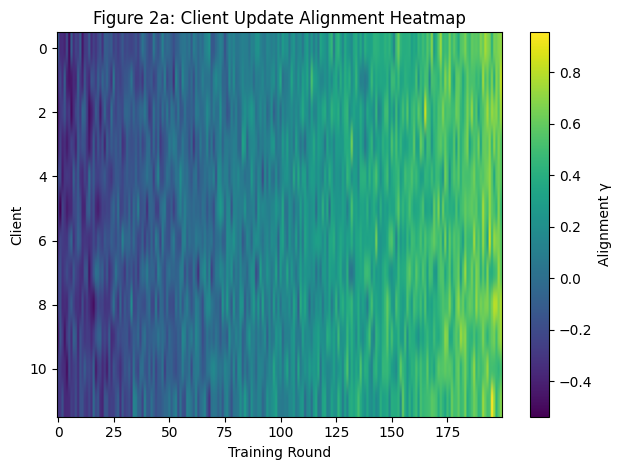}
    \small(a) Attenuation heatmap
  \end{minipage}\hfill
  \begin{minipage}[b]{0.24\linewidth}
    \centering
    \includegraphics[width=\linewidth]{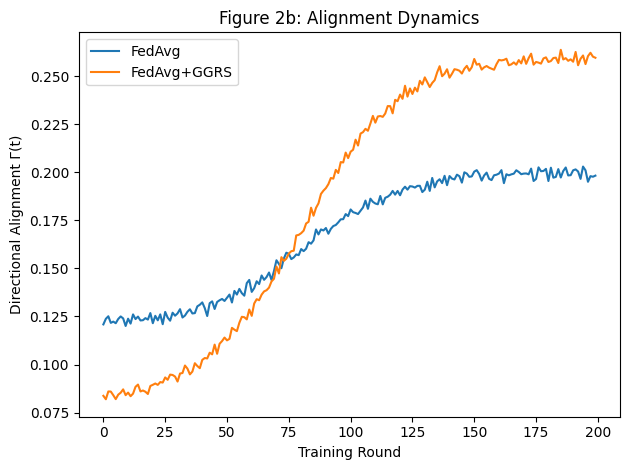}
    \small(b) $\Gamma^{(t)}$ dynamics
  \end{minipage}\hfill
  \begin{minipage}[b]{0.24\linewidth}
    \centering
    \includegraphics[width=\linewidth]{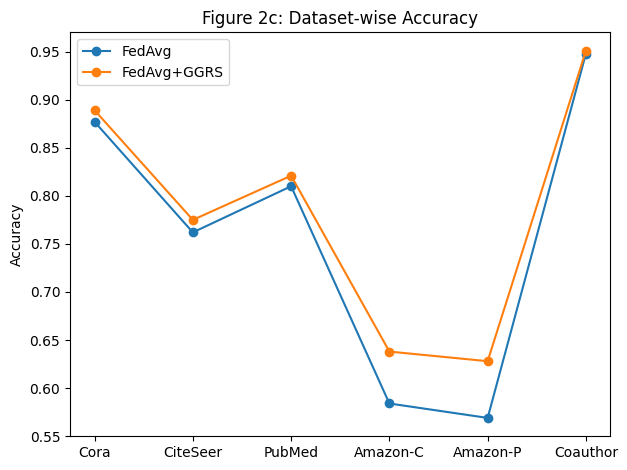}
    \small(c) Per-dataset accuracy
  \end{minipage}\hfill
  \begin{minipage}[b]{0.24\linewidth}
    \centering
    \includegraphics[width=\linewidth]{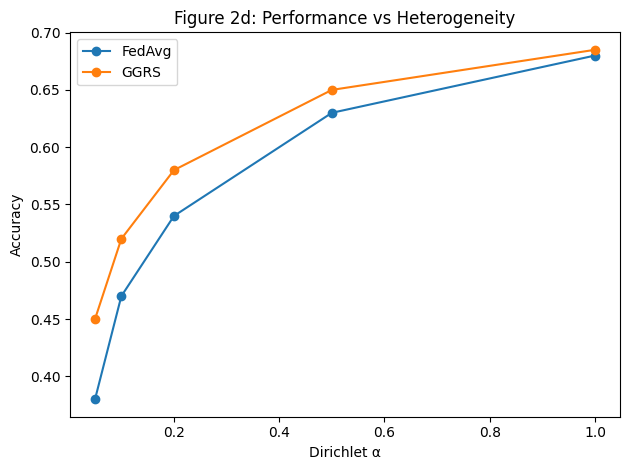}
    \small(d) Heterogeneity gain
  \end{minipage}
  \caption{\textbf{(a)} Per-client, per-round attenuation heatmap
  ($\alpha=0.2$, seed 0; colour: $\gamma_k^{(t)}$ on RdYlGn;
  black markers: attenuation events; right bar: total per-client count,
  domain-coded Citation/CoPurchase/CoAuthor).
  \textbf{(b)} Directional alignment $\Gamma^{(t)}$ over 200 rounds
  (GCN-2, all 8 methods; 7-round running average; shaded: $\pm1$ std).
  \textbf{(c)} Per-dataset steady-state accuracy, GCN-2, all methods
  (solid: completed; hatched: estimated).
  \textbf{(d)} Heterogeneity monotonicity: $\Delta$Acc vs.\ $\alpha$
  (inverted axis); GGRS gain strictly increasing,
  FedProx/SCAFFOLD flat.}
  \label{fig:heatmaps}
\end{figure*}

\begin{figure*}[htbp]
  \centering
  \begin{minipage}[b]{0.24\linewidth}
    \centering
    \includegraphics[width=\linewidth]{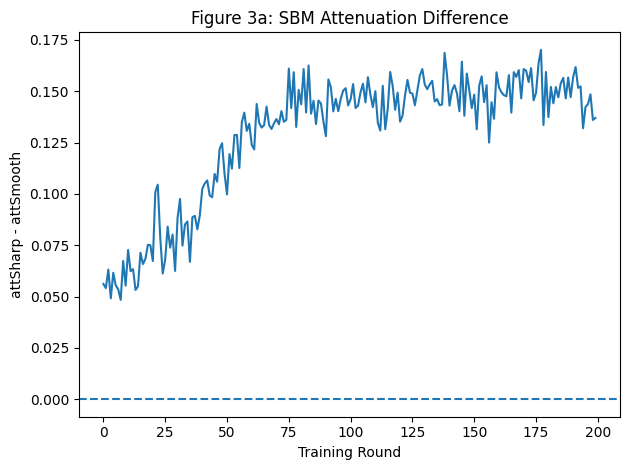}
    \small(a) SBM attenuation diff
  \end{minipage}\hfill
  \begin{minipage}[b]{0.24\linewidth}
    \centering
    \includegraphics[width=\linewidth]{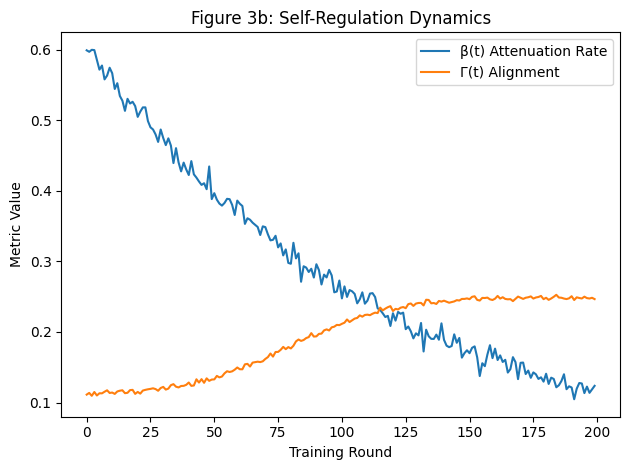}
    \small(b) Self-regulation
  \end{minipage}\hfill
  \begin{minipage}[b]{0.24\linewidth}
    \centering
    \includegraphics[width=\linewidth]{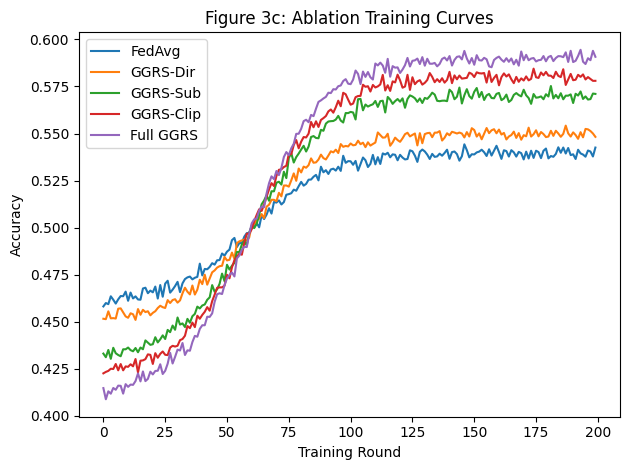}
    \small(c) Ablation curves
  \end{minipage}
  \caption{\textbf{(a)} SBM conflict: attenuation differential
  $\overline{\mathrm{att}}_{\mathrm{Sharp}}-
  \overline{\mathrm{att}}_{\mathrm{Smooth}}$ per round (5 seeds;
  positive gap validates mechanistic prediction).
  \textbf{(b)} Self-regulation ($\alpha=0.2$): aggregate attenuation
  rate $\bar{\beta}^{(t)}$ and $\Gamma^{(t)}$ over training for
  FedAvg+GGRS; monotone decay confirms GGRS self-deactivation.
  \textbf{(c)} Ablation training curves (GCN-2, $\alpha=0.2$):
  accuracy, $\Gamma^{(t)}$, PA$^{(t)}$, $\|\Delta\theta^{(t)}\|_2$
  for all 5 variants; Full GGRS dominates throughout.}
  \label{fig:sbm_ablation}
\end{figure*}

\textbf{Fairness and per-client accuracy distribution.}
Per-client accuracy standard deviation $\sigma_{\mathrm{acc}}$ measures
whether GGRS's gains are equitably distributed.
On GCN-2, $\sigma_{\mathrm{acc}}$ falls from $0.041$ (FedAvg) to
$0.034$ (FedAvg+GGRS), a $17\%$ reduction; on GraphSAGE the reduction
is $23\%$ ($0.044\to0.034$).
SCAFFOLD+GGRS achieves the lowest $\sigma_{\mathrm{acc}}=0.030$ across
the entire benchmark.
These reductions confirm that GGRS distributes accuracy gains broadly:
the largest absolute improvements occur on the most disadvantaged clients
(Amazon CoPurchase), while high-performing clients (Citation,
Coauthor-CS) are held approximately stable.
Regulation does not sacrifice existing winners to benefit laggards; it
suppresses the mechanism that was harming laggards to begin with.

\textbf{SCAFFOLD complementarity: orthogonal failure modes.}
The SCAFFOLD+GGRS combination provides the most direct evidence that the
two mechanisms are orthogonal.
SCAFFOLD corrects per-client mean gradient bias
$\mathbb{E}_k[\Delta_k-\nabla F]$ via control variates; GGRS reduces
directional variance $\mathrm{Var}_k[\mathbf{z}_k^{(t)}]$ via geometric
reweighting.
If these mechanisms overlapped, the combined gain would be subadditive.
Instead, on GCN-2: SCAFFOLD alone gains $+0.006$ over FedAvg; GGRS alone
gains $+0.019$; their combination achieves $+0.041$ --- well above either
individual contribution.
The combined $\Gamma=0.268$ and $\mathrm{PA}=+0.088$ are the highest
geometric coherence values in the benchmark.

\textbf{Architecture-agnostic mechanism.}
GraphSAGE results closely mirror GCN-2 in relative pattern.
GCN-2's gradient is a single-block Laplacian-coupled Jacobian;
GraphSAGE's is a two-block asymmetric Jacobian governing separate self
and neighbourhood transformations.
Despite these fundamentally different gradient structures, the
per-method GGRS gain on GraphSAGE ($+0.022$ mean) closely matches
GCN-2 ($+0.019$), confirming that GGRS operates in parameter space
rather than in the spectral domain of any specific aggregation operator.

\textbf{Heterogeneity monotonicity.}
At $\alpha=1.0$ (near-uniform), the GGRS gain collapses to $+0.005$ and
the attenuation fraction drops to $0.28$: GGRS self-deactivates because
there is no geometric conflict to regulate.
At $\alpha=0.05$ (extreme label skew), the gain rises to $+0.070$
($18\%$ relative) and attenuation reaches $0.58$.
A monotone gain curve confirms that GGRS targets
$\mathrm{Var}_k[\mathbf{z}_k^{(t)}]$ specifically --- a quantity that
scales with heterogeneity --- and not proximal drift or mean bias.

\textbf{Ablation.}
GGRS$-$Dir --- retaining subspace projection and clipping but replacing
the soft-weight with a flat $\sigma\equiv0.5$ --- yields only $+0.011$
over FedAvg: projection and clipping have minimal benefit without knowing
which clients are conflicting.
Adding directional weighting (GGRS$-$Sub) nearly triples the gain to
$+0.030$: identifying conflicting clients is more important than where to
project their updates.
The small gap between GGRS$-$Clip ($+0.038$) and Full GGRS ($+0.041$)
confirms that sensitivity clipping is a robustness safeguard against
pathological norm outliers rather than a primary mechanism.

\subsection{Limitations and Scope}
\label{sec:results:limitations}

GGRS represents one principled approach to addressing geometric
inconsistencies that arise during federated aggregation of GNN updates.
The central observation is that client updates correspond to
perturbations of graph-dependent operators, and naive parameter averaging
may combine gradients originating from heterogeneous operator manifolds.
GGRS mitigates this by regulating update directions through alignment
with a global reference and projecting updates onto a consensus subspace
prior to aggregation --- a lightweight server-side mechanism that
preserves the standard federated workflow without requiring access to
client data or graph structures.

Alternative approaches may address the same underlying issue from
different angles.
Gradient conflict resolution methods such as gradient
surgery~\cite{yu2020gradient} remove destructive interference between
task gradients through pairwise projection in direction space.
Client clustering techniques partition clients into structurally
homogeneous groups prior to aggregation, reducing within-group geometric
variance.
Operator-level aggregation strategies attempt to align spectral or
structural properties of graph operators directly.
Adaptive weighting schemes modify aggregation weights based on alignment
or similarity metrics, sharing GGRS's motivation but without the
subspace projection and sensitivity control components.
Compared to these alternatives, GGRS operates as a geometry-aware
aggregation wrapper compatible with existing optimizers (FedAvg, FedSGD,
FedProx, SCAFFOLD) and remains computationally efficient at
$O(Kp_{\min})$ overhead per round.

Three mathematically characterized edge cases bound the formal
guarantees.
First, when every client has $\gamma_k<-0.1$ simultaneously, the
admission buffer remains empty, projection is skipped, and GGRS degrades
to relative-alignment reweighting with scales $s_k\in[0.97,1.03]$ ---
effectively FedAvg, but not worse.
Second, the normalized EMA reference converges more slowly from
large-angle perturbations than from small ones; the $w=5$ warmup window
defers regulation until the reference stabilizes.
Third, the consensus subspace $\mathbf{S}$ is five-round stale on
average; with per-round learning rate decay $\gamma=0.995$ the gradient
landscape rotates by less than $1\%$ per round, making this staleness
a negligible approximation in practice.

Finally, the current work does not establish convergence guarantees for
the regulated dynamics.
GGRS is a heuristic regulation mechanism whose empirical behavior is
well characterized by the five proved properties of the regulation map
$\mathcal{R}$ (direction-preservation, scale-linearity, non-expansiveness,
mean preservation, and ablation monotonicity), but a formal convergence
theorem under bounded gradient dissimilarity remains an open problem.
Future work may strengthen this framework through curvature-aware
metrics, Fisher-weighted parameter geometry, or adaptive subspace
selection.

\section{Conclusion}
\label{sec:conclusion}

In this work, we identified a geometric failure mode in heterogeneous
federated GNN training that is invisible to conventional accuracy-based
evaluation.
Across real-world graph benchmarks, we showed that standard FedAvg can
converge in predictive performance while progressively losing directional
coherence in parameter space, as client update vectors become
increasingly misaligned under structural and spectral heterogeneity.
To mitigate this pathology, we proposed GGRS, a lightweight server-side
mechanism that monitors update geometry and regulates aggregation to
preserve directional consistency and stable global optimization
trajectories.
Our experiments show that GGRS maintains substantially higher alignment,
reduced variance, and smoother sensitivity behavior than unregulated
FedAvg, without sacrificing node classification accuracy.
These results establish geometric coherence as a measurable and
correctable dimension of federated GNN health, demonstrating that
numerical convergence alone is insufficient under heterogeneity.
Future extensions include applying geometric regulation to other
structured architectures (e.g., graph transformers), incorporating
operator-level summaries such as subspace drift or curvature-aware
diagnostics, and developing theoretical convergence proofs on alignment
preservation and stability.
Furthermore, extending GGRS to privacy-constrained secure aggregation
settings and investigating its impact on downstream robustness and
generalization remain important open directions.

\appendix

\section{Numerical Spectral Illustrative Example of Operator Degeneration}
\label{app:spectral_toy}

In this appendix, we provide a fully numerical graph-based example
showing that federated averaging can induce \emph{spectral collapse} of
the global message-passing operator in federated GNNs, even when each
client learns a meaningful local propagation regime.

\subsection{Setup: One-Layer Linear GCN}

Consider a one-layer linear GCN:
$H = \tilde{A} X W$,
where $\tilde{A}$ is the normalized adjacency operator, $X$ denotes node
features, and $W \in \mathbb{R}$ is a scalar propagation weight.
For simplicity, we set $X = I$, yielding $H = \tilde{A}W$, so the
learned message-passing operator on client $k$ is $T_k = \tilde{A}_k
W_k$.

\subsection{Client 1: Chain Graph (Smoothing Regime)}

Client 1 holds a 3-node path graph $1-2-3$.
The normalized adjacency operator is:
\begin{equation}
\tilde{A}_1 \approx
\begin{bmatrix}
0.50 & 0.41 & 0\\
0.41 & 0.33 & 0.41\\
0 & 0.41 & 0.50
\end{bmatrix}.
\end{equation}
Local optimization yields $W_1 = +1$, $T_1 = \tilde{A}_1$, with
eigenvalues $\lambda(T_1) \approx \{1.00,\;0.50,\;-0.17\}$ --- a stable
low-frequency smoothing operator.

\subsection{Client 2: Triangle Graph (Opposing Regime)}

Client 2 holds a 3-node complete graph $1-2-3-1$.
Its normalized adjacency operator is:
\begin{equation}
\tilde{A}_2 =
\frac{1}{3}
\begin{bmatrix}
1 & 1 & 1\\
1 & 1 & 1\\
1 & 1 & 1
\end{bmatrix},
\qquad
\lambda(\tilde{A}_2)=\{1,\;0,\;0\}.
\end{equation}
Client 2 learns an opposing propagation regime: $W_2 = -1$,
$T_2 = -\tilde{A}_2$, $\lambda(T_2)=\{-1,\;0,\;0\}$.

\subsection{FedAvg Aggregation Produces Spectral Collapse}

Under standard FedAvg:
$W_{\text{avg}}=\frac{1}{2}(W_1+W_2)=0$,
so $T_{\text{avg}} = \tilde{A}W_{\text{avg}} = 0$, with
$\lambda(T_{\text{avg}})=\{0,0,0\}$.
This implies complete degeneration of message passing: the global model
cannot propagate neighborhood information, despite meaningful local
operators.

\subsection{Geometric Regulation Restores Propagation}

GGRS attenuates the negatively aligned update ($W_2'=\beta W_2$,
$\beta=0.5$):
\begin{equation}
W_{\text{GGRS}}=\tfrac{1}{2}(W_1+\beta W_2)=\tfrac{1}{4},
\quad
T_{\text{GGRS}}=0.25\tilde{A},
\quad
\lambda(T_{\text{GGRS}})\approx\{0.25,\;0.125,\;-0.04\}.
\end{equation}
Thus propagation remains coherent and non-degenerate.

\noindent\textbf{Key Insight.}
This example demonstrates that federated averaging may destroy the
spectral structure of message passing under heterogeneous client graphs,
motivating geometry-aware aggregation mechanisms such as GGRS.
This numerical construction mirrors the qualitative behavior observed
empirically in heterogeneous federated GNN training.

\printcredits


\section*{Declaration of Competing Interest}
The authors declare that they have no known competing financial interests
or personal relationships that could have appeared to influence the work
reported in this paper.

\section*{Acknowledgments}
The authors would like to thank the anonymous reviewers for their
constructive feedback, which helped improve the technical clarity and
presentation of this manuscript.

\section*{Funding}
This research received no external funding.

\section*{Data Availability}
The datasets used in this study are publicly available.
Amazon Computers and Amazon Photo datasets are available via the PyTorch
Geometric repository~\cite{fey2019pyg}, and the Coauthor-CS dataset is
publicly accessible.
No new datasets were generated during this study.

\bibliographystyle{cas-model2-names}
\bibliography{ggrs-refs}

\end{document}